%% file: anonymous-submission-latex-2026.tex
\documentclass[letterpaper]{article} 
\usepackage{aaai2026}  
\usepackage{times}  
\usepackage{helvet}  
\usepackage{courier}  
\usepackage[hyphens]{url}  
\usepackage{graphicx} 
\usepackage{natbib}  
\usepackage{caption} 
\usepackage{algorithm}
\usepackage{algorithmic}
\usepackage{amsthm}
\usepackage{xspace}

\usepackage{amsmath}
\usepackage{amssymb}
\usepackage{subcaption}
\usepackage[export]{adjustbox} 
\usepackage{xcolor}

\newtheorem{theorem}{Theorem}[section]
\newtheorem{lemma}{Lemma}[section]



\newcommand{\robot}[1]{R^{#1}}

\newcommand{\robcfg}[2]{q^{#1}_{#2}}
\newcommand{\robcfgvel}[2]{\dot{q}^{#1}_{#2}}

\newcommand{\robcfgspace}[1]{\mathcal{Q}_{\robot{}}^{#1}}
\newcommand{\robtraj}[2]{\tau^{#1}_{#2}}

\newcommand{\omcbsa}{$\Omega$\textsc{-CBSA}\xspace}
\newcommand{\omastar}{$\Omega$\text{-A}$\!^\star$\xspace}

\newcommand{\omastarsingle}{$\Omega$\text{-A}$\!^\star_\text{  single}$\xspace}

\newif\ifshowedits
\showeditstrue
\newcommand{\edit}[1]{\ifshowedits\textcolor{black}{#1}\else#1\fi}

\usepackage{newfloat}
\usepackage{listings}
\DeclareCaptionStyle{ruled}{labelfont=normalfont,labelsep=colon,strut=off} 
\floatstyle{ruled}
\newfloat{listing}{tb}{lst}{}
\floatname{listing}{Listing}
\title{Embodying Multi-Hand Manipulation Policies by Searching the Assignment and Null Spaces}
\author {
    \edit{Yorai Shaoul,
    Jiaoyang Li,
    Maxim Likhachev}
}
\affiliations {
    \edit{The Robotics Institute, Carnegie Mellon University}
}

\usepackage{bibentry}

\begin{document}

\maketitle
\begin{abstract}

Learned manipulation policies increasingly predict motions for abstract “hands” and are attractive in practice because they rely on easily collected demonstrations and transfer across robot platforms. Executing these trajectories on multi-arm robots, however, is not trivial. Multi-hand policy outputs must be assigned to physical arms, each arm must realize a configuration-space motion that tracks its prescribed end-effector trajectory, and all arms must respect kinematic limits and avoid collisions. In the absence of algorithms that directly address this problem, practitioners typically extend single-arm inverse-kinematics (IK) pipelines in an ad hoc way, with no guarantees of feasibility or safety. In this work, we close this execution gap with a search-based framework that is theoretically complete for grounding policy-generated multi-hand trajectories onto physical multi-arm systems. Building on Conflict-Based Search, our method explicitly searches over both the discrete assignment of trajectories to arms and the continuous Jacobian null spaces of redundant manipulators, using redundancy to avoid inter-arm collisions while tracking the prescribed motions. This unified treatment of assignment and null-space motion yields a practically efficient planner that safely realizes coordinated manipulation-policy outputs on multi-arm robots. \edit{See omcbsa.github.io for more.}
\end{abstract}


\input{LaTeX/introduction}
\input{LaTeX/background}

\input{LaTeX/method}
\input{LaTeX/experiments}
\input{LaTeX/conclusion}
\bibliography{aaai2026}


\end{document}

%% file: LaTeX/introduction.tex
\vspace{-0.2cm}
\section{Introduction}

\edit{Modern manipulation policies increasingly output end-effector motions rather than joint commands, making demonstrations easy to collect and enabling transfer across robot platforms \cite{black2024pizero, chi2024umi, chi2023diffusionpolicy}.} This abstraction works well for single-arm systems, where differential IK can reliably track the policy’s motions. In multi-arm settings, however, it introduces an execution gap: policies do not specify which arm should execute each motion, nor how several manipulators should follow prescribed paths without colliding. As a result, standard single-arm strategies fail even when multi-arm realizations exist.


We close this gap by casting multi-arm policy embodiment as a search over assignments and null-space motions, ensuring that each arm tracks a prescribed trajectory while avoiding inter-arm collisions. Succinctly, we contribute \edit{(1) an efficient planner, with two effective variants, for safely realizing multi-hand trajectories on multi-arm systems; (2) a theoretical analysis establishing resolution-completeness over the search space induced by motion primitives; and (3) a physical and simulated experimental study.}

\begin{figure}[t]
    \centering
    \hfill
    \begin{minipage}{0.24\linewidth}
        \includegraphics[width=\linewidth]{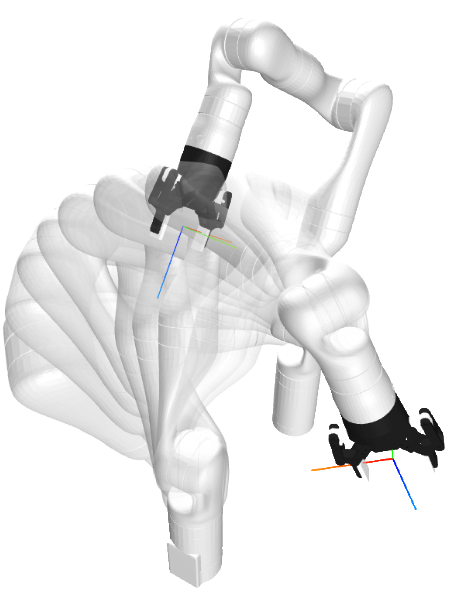}
    \end{minipage}
    \hfill
    \begin{minipage}{0.27\linewidth}
        \includegraphics[width=\linewidth]{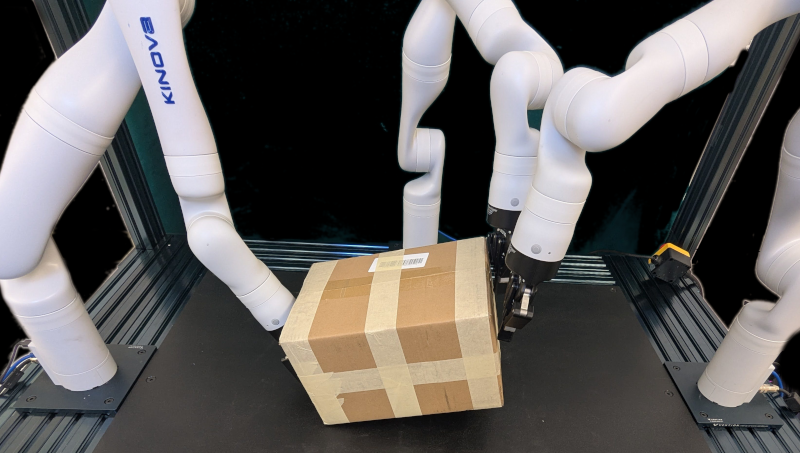}
        \includegraphics[width=\linewidth]{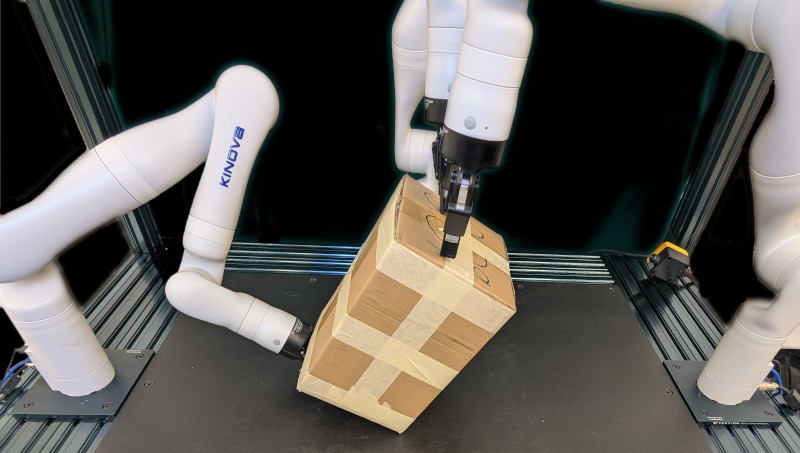}
    \end{minipage}
    \begin{minipage}{0.3\linewidth}
        \includegraphics[width=\linewidth]{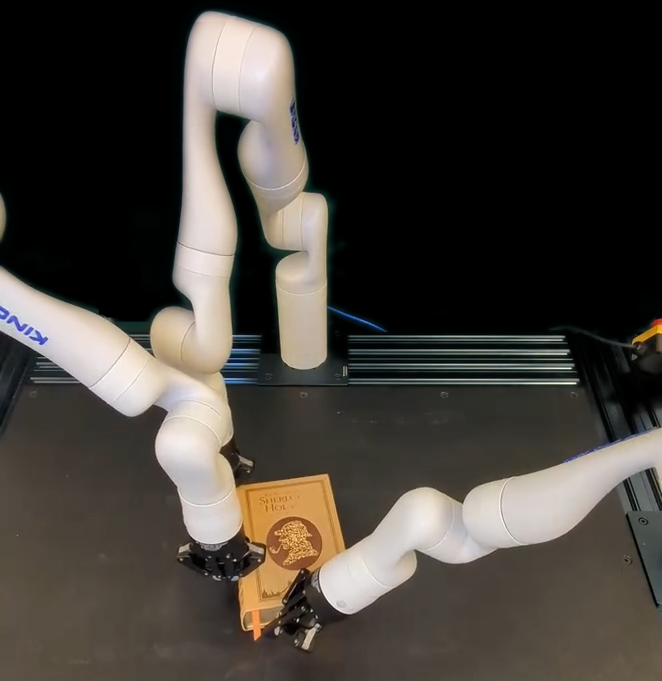}
    \end{minipage}

    \caption{Policy outputs often specify only end-effector trajectories, leaving a gap to determine arm--motion assignments and to reason over redundancy for collision-free embodiment. \textbf{Left:} visualizing the redundancy manifold $\mathcal{M}$ with multiple configurations $\robcfg{}{}$, some safe and some in collision, all with the same end-effector pose $x \in SE(3)$, i.e., $FK(\robcfg{}{}) = x$. \textbf{Right:} three arms collaboratively flip or push.}
    \label{fig:teaser}
    \vspace{-0.6cm}
\end{figure}

%% file: LaTeX/background.tex
\section{Background}
Let us begin by formally defining the problem at hand, reviewing relevant fundamentals, and surveying related work.

\subsection{On-Manifold Anonymous Multi-Robot-Arm Motion Planning (OM-AMRAMP)}
\label{sec:problem_formulation}

We consider $N$ manipulators $\{\robot{i}\}_{i=1}^N$, each with $d$ joints (also termed DOF: degrees of freedom) and configuration space $\robcfgspace{i} \subset \mathbb{R}^d$. A configuration $\robcfg{i}{} \in \robcfgspace{i}$ is a choice of joint angles for $\robot{i}$, which determines the pose of its end-effector (i.e., ``hand'') via the forward-kinematics map $FK^i : \robcfgspace{i} \rightarrow SE(3)$ and the volume it occupies in the workspace, denoted $\robot{i}(\robcfg{i}{}) \subset \mathbb{R}^3$. We denote the (static) obstacle region with  $\mathcal{O} \subset \mathbb{R}^3$.

The input to OM-AMRAMP is a set of $N$ end-effector trajectories of $H$ steps,
$X^j = \{x^j_1, \dots, x^j_H\}$ for $j = 1,\dots,N$, with $x^j_t \in SE(3)$, which must be realized by the arms. We seek a \textit{feasible} solution--a bijection $\sigma : \{1,\dots,N\} \to \{1,\dots,N\}$ from robots to trajectories and configuration-space trajectories
$\robtraj{i}{} = \{\robtraj{i}{1}, \dots, \robtraj{i}{H}\}$ with $\robtraj{i}{t} \in \robcfgspace{i}$, such that for all robots $\robot{i}$, timesteps $t$, and all $k \neq i$, each arm exactly tracks a trajectory without colliding. I.e.,
\[
FK^i(\robtraj{i}{t}) = x^{\sigma(i)}_t,\quad
\robot{i}(\robtraj{i}{t}) \cap \robot{k}(\robtraj{k}{t}) = \emptyset,\;
\robot{i}(\robtraj{i}{t}) \cap \mathcal{O} = \emptyset.
\]
We leave reaching $\robtraj{i}{1}$ to standard multi-arm motion planners \cite{mishani2025srmp, shaoul2024accelerating, shome2020drrt, huang2025apex-mr, RRT-Connect}. 

\subsection{Fundamentals and Related Work}
\label{sec:fundamentals}

The prevailing approach to embodying single-hand manipulation trajectories on a robotic arm is via inverse kinematics (IK) tracking \cite{yao2005ATACE}. This process typically begins by computing a first satisfying configuration $\robtraj{}{1}$ where $FK(\robtraj{}{1}) = x_1$, and then proceeding to find the remaining $\robtraj{}{t}$ such tha $FK(\robtraj{}{t}) = x_t$, often attempting to minimize $|\robtraj{}{t} - \robtraj{}{t-1}|$. Solving the IK problem, where we seek a configuration $\robcfg{}{}$ whose end-effector pose matches a desired $x \in SE(3)$, is often done numerically via the geometric Jacobian $J(\robcfg{}{})$, which relates joint velocities to end-effector velocities through $\dot{x} = J(\robcfg{}{})\robcfgvel{}{}$. Given an initial guess $\robcfg{}{\text{seed}}$ with pose error $\Delta x = x - FK(\robcfg{}{\text{seed}})$, a local IK update is obtained by $\Delta \robcfg{}{} = J^{+}(\robcfg{}{\text{seed}})\,\Delta x$, where $J^{+}$ is the pseudoinverse. Iterating $\robcfg{}{\text{seed}} \leftarrow \robcfg{}{\text{seed}} + \Delta \robcfg{}{}$ drives the end-effector toward the target. This process, which we denote $IK(x, \robcfg{}{\text{seed}})$, is used for tracking by attempting random seeds $\robcfg{}{\text{rand}}$ until finding $\robtraj{}{1} \gets IK(x_1, \robcfg{}{\text{rand}})$ and then $\robtraj{}{t} \gets IK(x_t, \robtraj{}{t-1})$.

For common manipulators with more than six joints, $FK(\robcfg{}{})= x_t$ admits infinitely many solutions (Fig. \ref{fig:teaser}), forming the low-dimensional manifold
$\mathcal{M}_t = \{\, \robcfg{}{} \in \robcfgspace{} \mid FK(\robcfg{}{}) = x_t \,\}$.
Because numerical IK returns only a \emph{single} solution determined by its seed, IK-tracking commits early to one trajectory solution and cannot explore configurations on $\mathcal{M}_t$ that might avoid collisions later in execution.

Advanced single-robot pose-constrained planners explore $\mathcal{M}_t$ more broadly, \edit{e.g., by sampling multiple IK solutions per pose and searching over their connections \cite{de2017descartes}, by projection-based sampling on the manifold \cite{stilman2007rrt-tc, berenson2009cbirrt, oriolo2005motion}, or by systematically searching redundant joint values \cite{cohen2014single_and_dual_arm}.} These methods improve upon IK-tracking, but they still do not address assignments and coordination jointly.


For redundant manipulators, a general way to explore $\mathcal{M}_t$ arises from the null space of the Jacobian. At a configuration $\robcfg{}{}$, the null space $\mathrm{Null}(J(\robcfg{}{}))$ consists of all joint-space directions that leave the end-effector pose unchanged.
I.e., with $N(\robcfg{}{})$, a matrix whose columns span $\mathrm{Null}(J(\robcfg{}{}))$, 
any update of the form $\Delta \robcfg{}{} = N(\robcfg{}{})\,\alpha$, where $\alpha$ is a small coefficient vector, moves the robot locally without altering its end-effector pose. These null-space directions provide systematic ``near neighbors’’ on the constraint manifold and allow local exploration of redundant configurations. In this work, we adopt this null-space-based exploration to generate feasible on-manifold motions while simultaneously reasoning about trajectory–arm assignments and inter-arm collisions.

%% file: LaTeX/method.tex
\section{On-Manifold CBS with Assignments}

We now introduce our method, \emph{On-Manifold Conflict-Based Search with Assignments} (\omcbsa, pronounced ``Om" CBSA). At its core, \omcbsa relies on a single-robot planner, \omastar, which performs search directly on the pose-constraint manifolds induced by a desired end-effector trajectory. We first present \omastarsingle for the single-trajectory setting and then extend it to a multi-goal variant that allows a robot to implicitly solve the assignment problem by selecting which trajectory to embody during search. Our main multi-robot algorithm, \omcbsa, builds on this multi-goal \omastar to jointly reason about trajectory--arm assignments and inter-arm collisions. \edit{We also introduce two prioritized planners (PP), OM-PP-A$\!^\star_\text{  single}$ and OM-PP-A$\!^\star$, using the single-goal and multi-goal variants of \omastar in a prioritized scheme \cite{erdmann1987prioritizedplanning}.}

\subsubsection{On-Manifold A$\!^\star$.}

\omastarsingle is a single-robot, single-hand-trajectory planner that performs A$\!^\star$ search \cite{a*} directly on the pose-constraint manifold induced by an end-effector path $X = \{x_1,\dots,x_H\}$. The algorithm maintains an \textsc{Open} list (a priority queue) of states $s = (\robcfg{}{},t)$, where $\robcfg{}{}$ is a configuration with $FK(\robcfg{}{}) = x_t$. \textsc{Open} is initialized with multiple IK solutions for $x_1$, obtained by running IK from diverse seeds to roughly sample the manifold $\mathcal{M}_1 = \{\robcfg{}{} \mid FK(\robcfg{}{})=x_1\}$. Each state stores a cost-to-come $g(s)$, a heuristic $h(s) = H - t$, a collision count $c(s)$ with respect to any provided other-robot paths, and a priority value $f(s) = g(s) + w_h h(s) + w_c c(s)$ as in \citet{weightedCBS}. At each step, \omastarsingle pops from \textsc{Open} the state with minimal $f(s)$; if this state is at $t = H$, the algorithm terminates and returns the corresponding on-manifold path. Otherwise, from a popped state $(\robcfg{}{},t)$, it generates successors $s'$ for $x_{t+1}$ by first computing a nominal projection $\robcfg{}{}{}' = IK(x_{t+1}, \robcfg{}{})$, and then exploring redundant realizations via nullspace motion: letting $N(\robcfg{}{}{}')$ denote a basis of $\mathrm{Null}(J(\robcfg{}{}{}'))$, for each basis direction $n \in N(\robcfg{}{}{}')$ and step size $\pm \epsilon$, the algorithm proposes $\robcfg{}{}{}'' = \robcfg{}{}{}' + \epsilon n$ followed by a correction step $\robcfg{}{}{}'' \leftarrow IK(x_{t+1}, \robcfg{}{}{}'')$ to return to the manifold. Each valid $\robcfg{}{}{}''$ (collision-free with obstacles and within joint limits) becomes a successor state $(\robcfg{}{}{}'', t+1)$ with updated\footnote{For initial states $g(s) = \Vert FK(q) - x_1\Vert$ for $q$ the current robot configuration. $c(s') = c(s) + \textsc{CountCollisions}(q')$.} $g(s') = g(s) + \|\robcfg{}{}{}''-\robcfg{}{}\|$, $h$, $c$, and $f$, and is inserted into \textsc{Open}. We set $\epsilon = 0.2, w_h=10, w_c=0.1$.

\subsubsection{On-Manifold Multi-Goal A$\!^\star$.}

To handle the setting where a robot can choose among multiple end-effector trajectories $\{X^1,\dots,X^M\}$, we design \omastar. Here, instead of initializing \textsc{Open} with IK solutions for the first pose of a single trajectory, we initialize it with IK solutions for the first pose of \emph{every} candidate trajectory. Each state now has the form $(q, t, j)$, encoding a timed configuration tracking $X^j$, and the search proceeds exactly as in \omastarsingle, with the only difference that successors $s' = (q', t+1, j)$ generated along  $X^j$. This has assignment decisions handled implicitly,
allowing the planner to determine both \emph{which} trajectory to embody and \emph{how} to track it without committing prematurely.

\subsubsection{On-Manifold Prioritized Planning.}

A simple way to coordinate multiple manipulators is to use \omastar as a low-level planner in a prioritized scheme. Robots are ordered by a priority and planned sequentially: the highest-priority robot plans first, and each subsequent robot runs \omastar (or \omastarsingle, if assignments are available) while treating previously planned trajectories as moving obstacles, disallowing configurations $q^i_t$ for which $\robot{i}(q^i_t) \cap \robot{k}(\robtraj{k}{t}) \neq \emptyset$ for some higher-priority $\robot{k}$. We refer to the resulting planners as OM-PP-A$\!^\star$ and OM-PP-A$\!^\star_{\text{   single}}$. These on-manifold prioritized planners are efficient and often effective in practice, but—as in standard prioritized schemes—are incomplete.

\subsubsection{On-Manifold Conflict-Based Search with Assignments.}

To obtain a complete multi-robot planner for pose-constrained policy embodiment, we integrate our multi-goal \omastar with Conflict-Based Search (CBS)~\cite{sharon2015conflict}. \omcbsa begins by planning an on-manifold trajectory for each robot\footnote{In practice, the root-node construction can be warm-started with repeated IK-tracking, with all later planning using \omastar.}; these initial trajectories may contain two types of conflicts: \emph{assignment conflicts}, in which two robots select the same candidate trajectory $X^j$, and \emph{geometric conflicts}, in which two robots collide at time $t$.

As in CBS, \omcbsa organizes the search in a constraint tree (CT). Each CT node contains (i) a set of per-robot constraints, (ii) the trajectories $\{\robtraj{i}{}\}_{i=1}^N$ computed under those constraints using \omastar, and (iii) the conflicts detected in these trajectories. At each iteration, \omcbsa selects the CT node with the fewest conflicts; if the node has no conflicts, its trajectories are returned as a solution. Otherwise, \omcbsa resolves one conflict by generating two child CT nodes, each created by copying the parent node and adding a different constraint. Each child then replans one robot, the robot to which the new constraint applies, using \omastar.

When an assignment conflict is encountered (e.g., robots $\robot{i}$ and $\robot{k}$ both select trajectory $X^j$) \omcbsa adds \emph{assignment constraints} (inspired by~\citet{motes2020tmpcbs}). One child forbids $\robot{i}$ from selecting $X^j$, and the other forbids $\robot{k}$. In the corresponding multi-goal \omastar replanning call, this is enforced by omitting all IK initializations for the forbidden trajectory when constructing the \textsc{Open} list.

If no assignment conflicts remain but a geometric conflict exists—say robots $\robot{i}$ and $\robot{k}$ collide at time $t$—\omcbsa selects a collision point $p \in \robot{i}(\robtraj{i}{t}) \cap \robot{k}(\robtraj{k}{t})$ and creates two child CT nodes: one forbidding robot $\robot{i}$ from occupying $p$ at time $t$, and one forbidding robot $\robot{k}$ from the same. In practice, during low-level validation, a state $(q,t)$ violates such a constraint if the occupied volume of the robot at timestep $t$ contains the witness point $p$. This is the same constraint model used in \cite{shaoul2024gencbs}

By combining CBS’s high-level reasoning over assignment and geometric constraints with low-level on-manifold feasibility checks, \omcbsa provides a complete planner for pose-constrained multi-arm policy embodiment.

\begin{figure*}[t]
    \centering
    \includegraphics[height=2.8cm,valign=t]{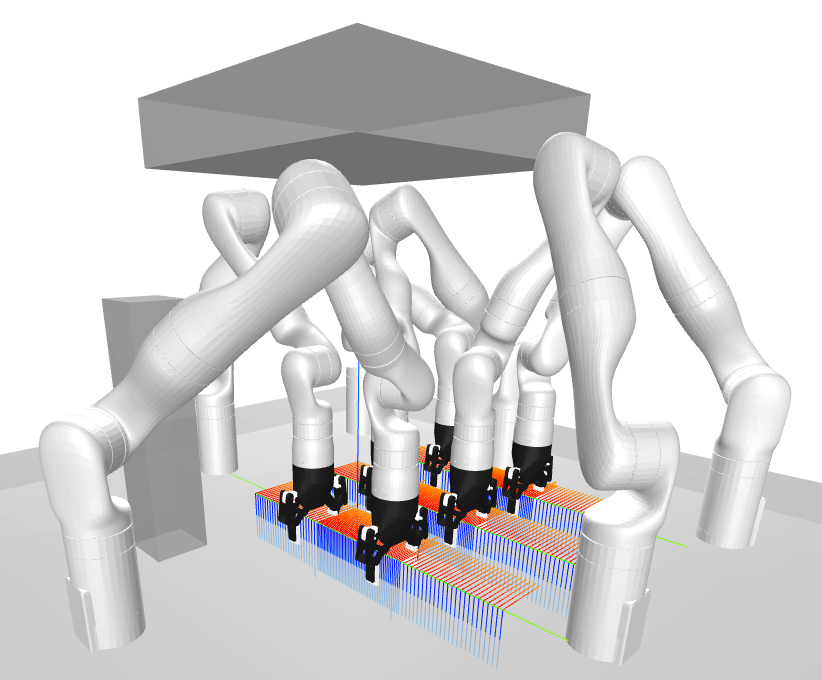}
    \includegraphics[height=2.8cm,valign=t]{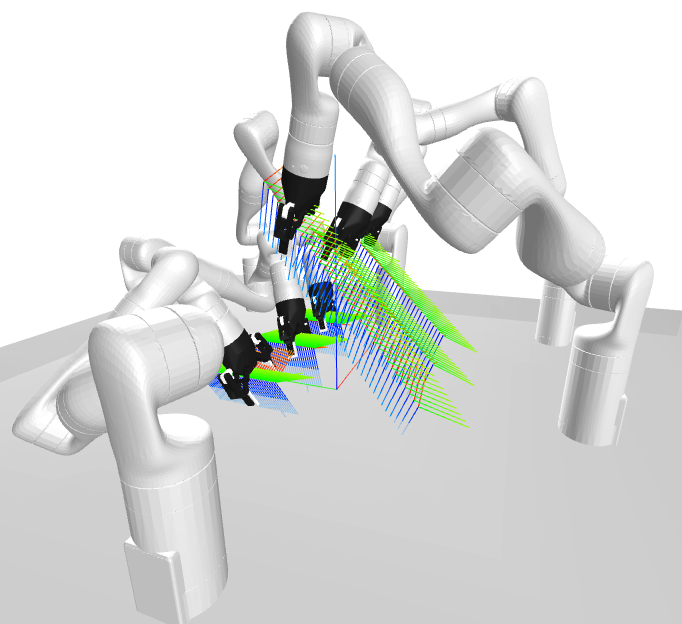}
    \includegraphics[width=0.6\linewidth,valign=t]{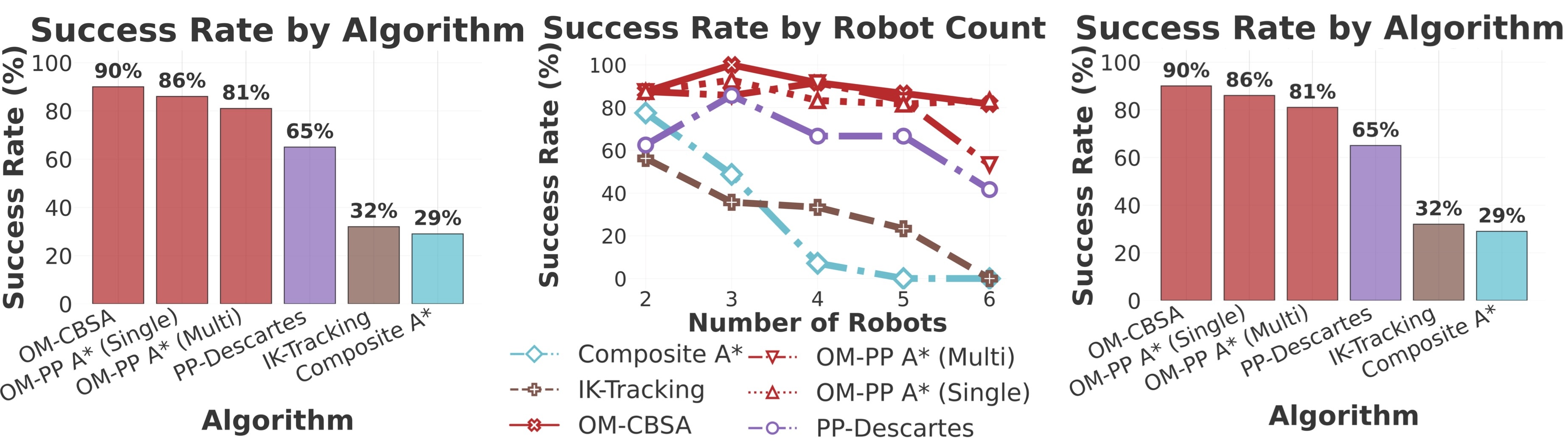}
    
    \caption{\textbf{Left images:} illustrations of our test cases visualizing six robots and six end-effector trajectories. The left image shows all robots moving from right to left in a constrained space, and the right image shows six robots performing a box-flipping motion. \textbf{From middle to the right:} overall success rates, success rates by robot count, and average overall runtime.}
    \label{fig:exps}
\end{figure*}

\subsection{Theoretical Analysis}
As is customary for search-based planning algorithms, \omcbsa adopts a finite-resolution discrete-time abstraction of the underlying continous problem: low-level motion is generated with a finite set of motion primitives. This abstraction allows us to establish the \textit{resolution}-completeness of \omcbsa for OM-AMRAMP. We do so by first analyzing \omastar, showing that it is resolution-complete, and then showing that its integration into a CBS framework with assignment and geometric constraints maintains this guarantee.

\begin{lemma}[Resolution-Completeness of \omastar] \label{lem:omastar}
If a feasible pose-constrained trajectory $\robtraj{i}{}$ following a fixed end-effector path $X^j \in \{X^1, \dots X^N\}$ exists for a $\robot{i}$ under motion primitive discretization resolution, then \omastar will find it.
\end{lemma}

\begin{proof}

\omastar searches a finite graph whose nodes are triples $(q, t, j)$ with $t \in \{1,\dots,H\}$, $q$ drawn from the finite set of configurations generated by the initialization set and motion primitives, and
$j$ a trajectory index. Since \omastar systematically explores this finite graph, every reachable node will eventually be expanded. Therefore, if a feasible trajectory exists in this induced graph and starts from one of the initialized states, \omastar will eventually reach and return it.
\end{proof}

For a CBS-based planner to be complete, its single-robot planner must be complete, and the constraints it uses to resolve conflicts must be \emph{mutually disjunctive}: for any pair of constraints $c^1, c^2$ generated from a single conflict between $\robot{1}$ and $\robot{2}$, there does not exist a conflict-free joint solution in which both $c^1$ and $c^2$ are violated~\cite{li2019largeAgents}.

\begin{lemma}[Constraints in \omcbsa] \label{lem:constraints}
The assignment and geometric point constraints are mutually disjunctive.
\end{lemma}

\begin{proof}
Consider first an assignment conflict where $\robot{1}$ and $\robot{2}$ both select $X^j$, and the two child CT nodes add constraints
$c^1 = \{\text{$\robot{1}$ may not select $X^j$}\}$ and
$c^2 = \{\text{$\robot{2}$ may not select $X^j$}\}$.
Any joint solution that violates both $c^1$ and $c^2$ necessarily assigns $X^j$ to both $\robot{1}$ and $\robot{2}$. This contradicts the requirement that the assignment $\sigma$ in OM-AMRAMP is a bijection between robots and trajectories, and therefore such a solution is invalid and cannot be conflict-free. Hence the assignment constraints are mutually disjunctive.
Geometric point constraints are mutually disjunctive, as shown in \citet{shaoul2024gencbs}.
\end{proof}

\begin{theorem}[Resolution-Completeness of \omcbsa]
\omcbsa is resolution-complete for OM-AMRAMP.
\end{theorem}

\begin{proof}
By Lemma \ref{lem:constraints}, all constraints added by \omcbsa are mutually disjunctive, and by Lemma \ref{lem:omastar}, the low-level planner \omastar is resolution-complete. Since CBS with a complete low-level planner and mutually disjunctive constraints is  complete, \omcbsa is resolution-complete.
\end{proof}

This establishes that \omcbsa fully reasons over assignments, redundant-joint realizations, and inter-arm collisions, and will succeed whenever a valid embodiment exists.

%% file: LaTeX/experiments.tex
\section{Experimental Analysis}

We simulated 450 benchmark problems to stress assignment reasoning, redundant-joint coordination, and collision avoidance. As illustrated in Fig.~\ref{fig:exps}, each problem comprised \edit{up to six} end-effector trajectories and robots arranged in diverse layouts with obstacles. The test suite included motions resembling learned manipulation-policy outputs and humanoid-inspired setups requiring tightly coordinated execution~\cite{black2024pizero,shaoul2025gco}. We also demonstrated \omcbsa on a physical 3-arm setup performing cloth rotation, box flip, and planar pushing (Fig.~\ref{fig:teaser}).

\subsection{Baselines}
To the best of our knowledge, no algorithm directly tackles the OM-AMRAMP problem. As such, we compared \omcbsa\ to three families of methods adapted from existing work: IK-tracking, pose-constrained roadmaps, and multi-robot implicit graph search.

\paragraph{IK-Tracking.}
This baseline reflects the standard practical workflow. Assuming a given assignment $\sigma$, each robot samples an IK solution for the first pose on their associated $X^j$ and tracks the remainder of the trajectory using IK. Our implementation iterates over all $\sigma$, proceeding to the next assignment if one fails. This approach is efficient but myopic.

\paragraph{PP-Descartes.}
We adapt Descartes~\cite{de2017descartes} to multi-robot settings. As with IK-Tracking, we enumerate all robot–trajectory assignments, and for each assignment we iterate over all priority orderings. Robots then plan in priority order: each constructs a pose-constrained roadmap by sampling IK solutions and connecting them with collision-free edges, requiring all samples and edges to also avoid the already planned higher-priority robots' trajectories. If all robots successfully build and search their roadmaps, a joint solution is returned; otherwise, the method attempts the next ordering or assignment. 

\paragraph{Composite A$\!^\star$.}
Inspired by search on implicit pose-constrained graphs~\cite{cohen2014single_and_dual_arm}, we implement a ``composite'' A$\!^\star$ planner whose states are joint tuples $s = (q^1,\dots,q^N,t)$. For each robot $\robot{i}$ in state $s$, we define its single-robot successor set $\mathcal{S}^i$ exactly as in \omastar, i.e., using an IK step and null-space exploration to move from time $t$ to $t+1$ along its assigned trajectory. The successor set of $s$ is then the Cartesian product of $\{\mathcal{S}^i\}_{i=1}^N$ at time $t+1$, restricted to collision-free joint configurations. The \textsc{Open} list is initialized with IK seeds for the first pose of all end-effector trajectories, implicitly covering all assignments. While expressive, this search now reasons over a state space that grows exponentially with the number of robots.

\subsection{Experimental Results}

We focused our analysis on success rate (fraction of problems solved within 5 seconds) and planning time.
Across all 450 benchmark problems, our proposed planners consistently outperformed all baselines in success rate. As shown in Fig.~\ref{fig:exps}, \omcbsa achieves the highest overall success, followed closely by its prioritized variants OM-PP-A$^{\!\star}$ and OM-PP-A$^{\!\star}_{\text{single}}$. We attribute this to the combination of lightweight single-robot planning over redundant configurations and flexible reasoning over trajectory--arm assignments. When disaggregated by robot count (Fig.~\ref{fig:exps}, second from right), the scalability gap becomes clearer: as the number of arms increases, the proposed methods remain consistently more successful than all baselines. While performing generally well, we observed that Descartes' performance notably suffers due to the cost of roadmap construction. Composite A$\!^\star$ suffers from rapid growth in branching factor as assignments are introduced, making search increasingly difficult. IK-Tracking, while simple and fast when it succeeds, is highly myopic and rarely found feasible solutions.

Our proposed planners remain fast: \omcbsa solved two-robot problems in $150$\,ms on average, and across all benchmarks achieved \edit{a mean runtime of roughly $640$\ ($\sigma = 990$\,ms, min $16$\,ms, max $4.9$\,s}, Fig.~\ref{fig:exps}, rightmost plot). Our prioritized variants solved most problems in under one second. \edit{Among baselines, IK-Tracking was the fastest when it succeeded (about $160$\,ms on average), but its low success rate limits its practical utility. Overall, these results demonstrate that \omcbsa can solve more difficult planning problems than baselines, with the prioritized variants offering favorable speed--performance tradeoffs.}

\paragraph{Implementation details.}
\edit{We used Pinocchio~\cite{carpentier2019pinocchio} for kinematics, Kinova Gen3 arms, and an Intel Core i9-12900H CPU (5.2\,GHz).} Within \omcbsa, root CT nodes were initialized using repeated IK-tracking (capped at $500$\,ms), producing initial trajectories with \edit{fewer} collisions without affecting theoretical guarantees.

%% file: LaTeX/conclusion.tex
\section{Conclusion}
We introduced \omcbsa, a theoretically complete planner for embodying multi-hand manipulation policies on multi-arm robots. By combining on-manifold single-robot planning with a CBS-based coordination scheme, our method systematically reasons about redundancy, collisions, and assignments. Experiments with up to six arms show that \omcbsa and its prioritized variants achieve higher success rates than IK, sampling-based, and implicit graph-search baselines while maintaining real-time performance. We hope this work will lower the practical barriers associated with deploying multi-hand manipulation policies.